\documentclass{article}

\usepackage{times}
\usepackage{graphicx} % more modern
\usepackage{subfigure} 
\usepackage{balance}

\usepackage{natbib}

\usepackage{algorithm}
\usepackage{algorithmic}

\usepackage[dvipsnames]{xcolor}
\usepackage{hyperref}
\hypersetup{
  colorlinks=true,
  linkcolor=red,
  urlcolor=blue,
  citecolor=blue
}
\usepackage{xr}
\usepackage[accepted]{icml2017}

\icmltitlerunning{Zonotope Hit-and-run for Efficient Sampling from Projection DPPs}

\usepackage{etoolbox}
\makeatletter
\patchcmd{\NAT@test}{\else \NAT@nm}{\else \NAT@nmfmt{\NAT@nm}}{}{}
\DeclareRobustCommand\citepos
  {\begingroup
   \let\NAT@nmfmt\NAT@posfmt% ...except with a different name format
   \NAT@swafalse\let\NAT@ctype\z@\NAT@partrue
   \@ifstar{\NAT@fulltrue\NAT@citetp}{\NAT@fullfalse\NAT@citetp}}
\let\NAT@orig@nmfmt\NAT@nmfmt
\def\NAT@posfmt#1{\NAT@orig@nmfmt{#1's}}

\usepackage{amsmath,amssymb,amsthm,dsfont,mdframed}
\usepackage{xspace}

\usepackage{framed}
\usepackage[strict]{changepage}    

\definecolor{mypink1}{rgb}{0.858, 0.188, 0.478}
\definecolor{mypink2}{RGB}{219, 48, 122}
\definecolor{mypink3}{cmyk}{0, 0.7808, 0.4429, 0.1412}
\definecolor{mygray}{gray}{0.8}
\definecolor{babyblue}{rgb}{0.54, 0.81, 0.94}
\definecolor{citrine}{rgb}{0.89, 0.82, 0.04}

\definecolor{formalshade}{rgb}{0.90,0.95,0.98}    

\newenvironment{formal}{%
\fontfamily{cmss}\selectfont
\fontsize{8}{10}\selectfont
  \MakeFramed{\advance\hsize-\width\FrameRestore}%
  \noindent\hspace{10cm}% disable indenting first paragraph
  \begin{adjustwidth}{}{5pt}%
  \vspace{-3pt}%
}
{%
  \vspace{2pt}\hspace{3pt}\end{adjustwidth}\endMakeFramed%
}

\def\twofig{.43\textwidth}
\def\threefig{.3\textwidth}

\newcommand\Vol{\operatorname{Vol}}

\usepackage{xspace}
\newcommand{\nystrom}{Nystr\"{o}m\xspace}

\newcommand\bA{\mathbf{A}}
\newcommand\bB{\mathbf{B}}
\newcommand\bC{\mathbf{C}}
\newcommand\bK{\mathbf{K}}
\newcommand\bL{\mathbf{L}}

\newcommand\Zon{\mathcal{Z}}
\newcommand\cB{\mathcal{B}}
\newcommand\cD{\mathcal{D}}
\newcommand\cO{\mathcal{O}}
\newcommand\cT{\mathcal{T}}

\newcommand{\lrb}[1]{\left[ #1 \right]}
\newcommand{\lrp}[1]{\left( #1 \right)}
\newcommand{\lrcb}[1]{\left\{ #1 \right\}}

\newcommand{\lrabs}[1]{\left| #1 \right|}
\newcommand\Prob{\operatorname{\mathbb{P}}}

\newcommand{\Proba}[1]{\Prob\lrb{#1}}

\newcommand\indic{\mathds{1}}
\newcommand\Unif{\operatorname{\mathcal{U}}}
\newcommand\R{\mathbb{R}}
\renewcommand{\top}{\mathsf{\scriptscriptstyle T}}

\renewcommand{\tilde}{\widetilde}

\newtheorem{assumption}[]{Assumption}
\newtheorem{proposition}[]{Proposition}
\newtheorem{remark}[]{Remark}

\begin{document} 

  \twocolumn[
  \icmltitle{Zonotope Hit-and-run for Efficient Sampling from Projection DPPs}



  % you can specify symbols, otherwise they are numbered in order
  \icmlsetsymbol{equal}{*}

  \begin{icmlauthorlist}
  \icmlauthor{Guillaume Gautier}{CNRS,INRIA}
  \icmlauthor{R\'emi Bardenet}{CNRS}
  \icmlauthor{Michal Valko}{INRIA}
  \end{icmlauthorlist}

  \icmlaffiliation{INRIA}{INRIA Lille --- Nord Europe, SequeL team}% , 40 avenue Halley 59650, Villeneuve d'€™Ascq, France}
  \icmlaffiliation{CNRS}{Univ.\,Lille, CNRS, Centrale Lille, UMR 9189 --- CRIStAL}%  - Centre de Recherche en Informatique Signal et Automatique de Lille, F-59000 Lille, France}
  % \icmlaffiliation{ed}{University of Edenborrow, Edenborrow, United Kingdom}

  \icmlcorrespondingauthor{Guillaume Gautier}{g.gautier@inria.fr}
  %\icmlcorrespondingauthor{R\'emi Bardenet}{remi.bardenet@gmail.com}
  %\icmlcorrespondingauthor{Michal Valko}{michal.valko@inria.fr}

  \icmlkeywords{DPPs, MCMC, zonotopes, machine learning, ICML}

  \vskip 0.3in
  ]



  \printAffiliationsAndNotice{}  % leave blank if no need to mention equal contribution

  \begin{abstract} 
    Determinantal point processes (DPPs) are distributions over sets of items that model diversity using kernels. Their applications in machine learning include summary extraction and recommendation systems.
    Yet, the cost of sampling from a DPP is prohibitive in large-scale applications, which has triggered an  effort towards efficient approximate samplers. 
    We build a novel MCMC sampler that combines ideas from combinatorial geometry, linear programming, and Monte Carlo methods to sample from DPPs with a fixed sample cardinality, also called projection DPPs. 
    Our sampler leverages the ability of the hit-and-run MCMC kernel to efficiently move across convex bodies. 
    Previous theoretical results yield a fast mixing time of our chain when targeting a distribution that is close to a projection DPP, but not a DPP in general. 
    Our empirical results demonstrate that this extends to sampling projection DPPs, i.e., our sampler is more sample-efficient than previous approaches 
    which in turn translates to faster convergence when dealing with costly-to-evaluate functions, such as summary extraction in our experiments. 
  \end{abstract} 

  \section{Introduction}
  \label{s:introduction}

    Determinantal point processes (DPPs) are distributions over configurations of points that encode diversity through a kernel function. DPPs were introduced by \citet{Mac75} and have then found applications in fields as diverse as probability \citep{HKPV06}, number theory \citep{RuSa96}, statistical physics \citep{PaBe11}, Monte Carlo methods \citep{BaHa16Sub}, and spatial statistics \citep{LaMoRu15}. 
    In machine learning, DPPs over finite sets have been used as a model of diverse sets of items, where the kernel function takes the form of a finite matrix, see \citet{KuTa12} for a comprehensive survey. 
    Applications of DPPs in machine learning (ML) since this survey also include recommendation tasks \citep{KaDeKo16,GaPaKo16}, text summarization \citep{DuBa16}, or models for neural signals \citep{SnZeAd13}. 

    Sampling generic DPPs over finite sets is expensive. 
    Roughly speaking, it is cubic in the number $r$ of items in a DPP sample. 
    Moreover, generic DPPs are sometimes specified through an $n\times n$ kernel matrix that needs diagonalizing before sampling, where $n$ is the number of items to pick from. 
    In text summarization, $r$ would be the desired number of sentences for a summary, and $n$ the number of sentences of the corpus to summarize. Thus, sampling quickly becomes intractable for large-scale applications \citep{KuTa12}. 
    This has motivated research on fast sampling algorithms. 
    While fast exact algorithms exist for specific DPPs such as uniform spanning trees \citep{Ald90,Bro89,PrWi98JoA}, generic DPPs have so far been addressed with approximate sampling algorithms, using random projections \citep{KuTa12}, low-rank approximations \citep{KuTa11, GiKuTa12,AKFT13}, or using Markov chain Monte Carlo techniques \citep{Kan13,LiJeSr15,ReKa15,AnGhRe16,LiJeSr16}. 
    In particular, there are polynomial bounds on the mixing rates of natural MCMC chains with arbitrary DPPs as their limiting measure; see \citet{AnGhRe16} for cardinality-constrained DPPs, and \citet{LiJeSr16} for the general case. 

    In this paper, we contribute a non-obvious MCMC chain to approximately sample from \emph{projection DPPs}, which are DPPs with a fixed sample cardinality. 
    Leveraging a combinatorial geometry result by \citet{DyFr94}, we show that sampling from a projection DPP over a finite set can be relaxed into an easier continuous sampling problem with a lot of structure. 
    In particular, the target of this continuous sampling problem is supported on the volume spanned by the columns of the feature matrix associated to the projection DPP, a convex body also called a \emph{zonotope}. 
    This zonotope can be partitioned into tiles that uniquely correspond to DPP realizations, and the relaxed target distribution is flat on each tile. 
    Previous MCMC approaches to sampling projections DPPs can be viewed as attempting moves between neighboring tiles. Using linear programming, we propose an MCMC chain that moves more freely across this tiling. 
    Our chain is a natural transformation of a fast mixing hit-and-run Markov chain \citep{LoVe03} on the underlying zonotope; this empirically results in more uncorrelated MCMC samples than previous work. 
    While the results of \citet{AnGhRe16} and their generalization by \citet{LiJeSr16} apply to projection DPPs, our experiments support the fact that our chain mixes faster.

    The rest of the paper is organized as follows. 
    In Section~\ref{s:dpp}, we introduce projection DPPs and review existing approaches to sampling. 
    In Section~\ref{s:zonotopes}, we introduce zonotopes and we tailor the hit-and-run algorithm to our needs. 
    In Section~\ref{s:experiments}, we empirically investigate the performance of our MCMC kernel on synthetic graphs and on a summary extraction task, before concluding in Section~\ref{s:discussion}.

  \section{Sampling Projections DPPs}
  \label{s:dpp}

    In this section, we introduce projection DPPs in two equivalent ways, respectively following \citet{HKPV06}, \citet{KuTa12}, and \citet{Lyo03}. 
    Both definitions shed a different light on the algorithms in Section~\ref{s:zonotopes}.

    \subsection{Projection DPPs as Particular DPPs}
    \label{s:hough}
      Let $E=[n]\triangleq\{1,\dots n\}$. 
      Let also $\bK$ be a real symmetric positive semidefinite $n\times n$ matrix, and for $I\subset E$, write $\bK_I$ for the square submatrix of $\bK$ obtained by keeping only rows and columns indexed by $I\subset E$. 
      The random subset $X\subset E$ is said to follow a DPP on $E=\{1,\dots,n\}$ with kernel $\bK$ if
      \begin{equation}
      \label{e:dpp}
        \Proba{I \subset X}= \det \bK_{I},\quad\forall I\subset E.
      \end{equation}

      Existence of the DPP described by \eqref{e:dpp} is guaranteed provided $\bK$ has all its eigenvalues in $[0,1]$, see e.g., \citet[Theorem 2.3]{KuTa12}. 
      Note that \eqref{e:dpp} encodes the repulsiveness of DPPs. 
      In particular, for any distinct $i,j \in [n]$,
      \begin{align*}
          \Proba{\{i,j\} \subset X}
          &=
          \begin{vmatrix}
              \mathbf{K}_{ii} & \mathbf{K}_{ij}\\
              \mathbf{K}_{ji} & \mathbf{K}_{jj}
          \end{vmatrix}\\
          &= \Proba{\{i\}\in X} \Proba{\{j\}\in X} - \mathbf{K}_{ij}^2\\ 
          &\leq \Proba{\{i\}\in X} \Proba{\{j\}\in X}.
      \end{align*}
      In other words, $\bK_{ij}$ encodes departure from independence. 
      Similarly, for constant $\bK_{ii},\bK_{jj}$, the larger $\bK_{ij}^2$, the less likely it is to have items $i$ and $j$ co-occur in a sample.

      \emph{Projection DPPs} are the DPPs such that the eigenvalues of $\bK$ are either $0$ or $1$, that is, $\bK$ is the matrix of an \emph{orthogonal} projection. 
      Projection DPPs are also sometimes called elementary DPPs \citep{KuTa12}.
      One can show that samples from a projection DPP with kernel matrix $\bK$ almost surely contain $r=\text{Tr}(\bK)$ points and that general DPPs are mixtures of projection DPPs, see e.g., \citet[Theorem 2.3]{KuTa12}.

    \subsection{Building Projection DPPs from Linear Matroids}
    \label{s:lyons}

      Let $r<n$, and let $\bA$ be a full-rank $r\times n$ real matrix with columns $(a_j)_{j\in[n]}$. 
      The linear matroid $M[\bA]$ is defined as the pair $(E,\cB)$, with $E=[n]$ and
      \begin{equation}
          \label{e:collection_of_bases-cB}
   \cB  = \!\left\{ B\! \subset  [n] :  | B | = r, \left\{a_j\right\}_{j\in B} \text{are independent} \right\}.
      \end{equation}
      A set of indices $B\subset[n]$ is in $\cB$ if and only if it indexes a basis of the columnspace of~$\bA$. 
      Because of this analogy, elements of $\cB$ are called \emph{bases} of the matroid $M[\bA]$. 
      Note that elementary algebra yields that for all $B_1, B_2 \in \cB$ and $x \in B_1 \setminus B_2$, there exists an element $y \in B_2 \setminus B_1$ such that 
      \begin{equation}
      \label{e:basisExchange}
          \left(B_1 \setminus \left\{x\right\}\right) \cup \{y\} \in \cB.
      \end{equation}
      Property~\eqref{e:basisExchange} is known as the \emph{basis-exchange} property. 
      It is used in the definition of general matroids \citep{Oxl03}.

      \citet{Lyo03} defines a projection DPP as the probability measure on $\cB$ that assigns to $B\in\cB$ a mass proportional to $|\det\bB|^2$, where $\bB \triangleq \bA_{:B}$ is the square matrix formed by the $r$ columns of $\bA$ indexed by $B$. 
      Note that this squared determinant is also the squared volume of the parallelotope spanned by the columns indexed by $B$. 
      In this light, sampling a projection DPP is akin to volume sampling \citep{DeRa10}.
      Finally, observe that the Cauchy-Binet formula gives the normalization
      $$ \sum_{B\in\cB}  \lrabs{\det \bA_{:B}}^2 = \det \bA\bA^\top, $$
      so that the probability mass assigned to $B$ is 
      \begin{equation*}
          \frac{ \det {\bA^{\top}}_{B:}
            \det\bA_{:B}}{ \det \bA\bA^{\top} } = \det
          \lrb{\bA^{\top} \lrb{\mathbf{AA}^{\top}}^{-1} \bA}_B.
      \end{equation*}
      Letting 
      \begin{equation}
      \label{e:projectionKernel}
          \bK=\bA^{\top} \lrb{\mathbf{AA}^{\top}}^{-1} \bA,   
      \end{equation}
      gives the equivalence between Sections~\ref{s:hough} and \ref{s:lyons}.

      A fundamental example of DPP defined by a matroid is the random set of edges obtained from a uniform spanning tree \citep{Lyo03}. 
      Let $G$ be a connected graph with $r+1$ vertices and $n$ edges $\{e_i\}_{i\in [n]}$. 
      Let now $\bA$ be the first $r$ rows of the vertex-edge incidence matrix of $G$. 
      Then $B\subset[n]$ is a basis of $M[\bA]$ if and only if $\{e_i\}_{i\in B}$ form a spanning tree of~$G$ \citep{Oxl03}. 
      The transfer current theorem of \citet{BuPe93} implies that the uniform distribution on~$\cB$ is a projection DPP, with kernel matrix \eqref{e:projectionKernel}. 

    \subsection{On Projection DPPs and $k$-DPPs in ML}
    \label{s:kdpps}

      Projection DPPs are DPPs with realizations of constant cardinality $k=r$, where $r$ is the rank of $\bK$. 
      This constant cardinality is desirable  when DPPs are used in summary extraction 
      \citep{KuTa12,DuBa16} and the size of the required output is predefined. 
      Another way of constraining the cardinality of a DPP is to condition on the event $\vert X\vert=k$, which leads to the so-called $k$-DPPs \citep{KuTa12}.
      Projection DPPs and $k$-DPPs are in general different objects. 
      In particular, a $k$-DPP is not a DPP in the sense of \eqref{e:dpp} unless its kernel matrix $\bK$ is a projection. In that sense, $k$-DPPs are non-DPP objects that generalize projection DPPs. 
      In this paper, we show that projection DPPs can benefit from fast sampling methods. 
      It is not obvious how to generalize our algorithm to $k$-DPPs. 

      In ML practice, using projection DPPs is slightly different from using a $k$-DPP. 
      In some applications, typically with graphs, the DPP is naturally a projection, such as uniform spanning trees described in Section~\ref{s:lyons}. 
      But quite often, kernels are built feature-by-feature. 
      That is, for each data item $i\in [n]$, a normalized vector of features $\phi_i\in\R^r$ is chosen, a marginal
      relevance $q_i$ is assigned to item $i$, and a matrix $\bL$ is defined as
      \begin{equation}
      \label{e:textKernel}
          \bL_{ij}=\sqrt{q_i}\phi_i\phi_j\sqrt{q_j}.
      \end{equation}
      In text summarization, for instance, items $i,j$ could be sentences, $q_i$ the marginal relevance of sentence $i$ to the user's query, and $\phi_i$ features such as tf-idf frequencies of a choice of words, and one could draw from a
      $k$-DPP associated to $\bL$ through $\Proba{X=I}\propto \det \bL_I$, see e.g., \citet[Section 4.2.1]{KuTa12}.

      Alternately, let $\bA$ be the matrix with columns $(\sqrt{q_i}\phi_i)_{i\in [r]}$, and assume $r<n$ and~$\bA$ is full-rank. 
      The latter can be ensured in practice by adding a small i.i.d.\,Gaussian noise to each entry of $\bA$. 
      The projection DPP with kernel $\bK$ in \eqref{e:projectionKernel} will yield samples of cardinality $r$, almost
      surely, and such that the corresponding columns of $\bA$ span a large volume, hence feature-based diversity. 
      Thus, if the application requires an output of length $p$, one can pick $r=p$, as we do in Section~\ref{s:app_mnist}.
      Alternatively, if we want an output of size approximately $p$, we can pick $r\geq p$ and independently thin the resulting sample, which preserves the DPP structure \citep{LaMoRu15}.
      
  \vspace{-0.5em}
    \subsection{Exact Sampling of Projection DPPs}
    \label{ss:exact}

      \citet{HKPV06} give an algorithm to sample general DPPs, which is based on a subroutine to sample projection DPPs. Consider a projection DPP with kernel $\bK$ such that $\text{Tr}(\bK)=r$, \citepos{HKPV06} algorithm follows the chain rule to sample a vector $(x_1,\dots,x_r)\in [n]^r$ with successive conditional densities 
      \begin{equation*}
        p\left(x_{\ell+1} = i\vert x_1=i_1,\dots,x_\ell=i_\ell\right) \propto \bK_{ii} - \bK_{i,I_\ell}\bK_{I_\ell}^{-1}\bK_{I_\ell,i},
      \end{equation*}
      where $I_{\ell}=\{i_1,\dots,i_\ell\}$. 
      Forgetting order, $\{x_1,\dots,x_r\}$ are a draw from the DPP \citep[Proposition 19]{HKPV06}, see also \citet[Theorem 2.3]{KuTa12} for a detailed treatment of DPPs on $[n]$.

      While exact, this algorithm runs in $\cO(nr^3)$ operations and requires computing  and storing the $n\times n$ matrix $\bK$. 
      Storage can be diminished if one has access to $\bA$ in \eqref{e:projectionKernel}, through QR decomposition of $\bA^{\top}$. 
      Still, depending on $n$ and $r$, sampling can become intractable. 
      This has sparked interest in fast approximate sampling methods for DPPs, which we survey in Section~\ref{s:approximate}.

      Interestingly, there exist \emph{fast} and \emph{exact} methods for sampling some specific DPPs, which are not based on the approach of \citet{HKPV06}. 
      We introduced the DPP behind uniform spanning trees on a connected graph $G$ in Section~\ref{s:lyons}.
      Random walk algorithms such as the ones by \citet{Ald90}, \citet{Bro89}, and \citet{PrWi98JoA} sample uniform spanning trees in time bounded by the cover time of the graph, for instance, which is $\cO(r^3)$ and can be $o(r^3)$ \citep{LePeWi09}, where $G$ has $r+1$ vertices. 
      This compares favorably with the algorithm of \citet{HKPV06} above, since each sample contains $r$ edges. The Aldous-Broder algorithm, for instance, starts from an empty set $\cT=\emptyset$ and an arbitrary node $x_0$, and samples a simple random walk $(X_t)_{t\in\mathbb{N}}$ on the edges of $G$, starting from $X_0=x_0$, and adding edge $[X_t,X_{t+1}]$ to $\cT$ the first time it visits vertex $X_{t+1}$. 
      The algorithm stops when each vertex has been seen at least once, that is, at the cover time of the graph.  

    \subsection{Approximate Sampling of Projection DPPs}
    \label{s:approximate}
      There are two main sets of methods for approximate sampling from general DPPs. 
      The first set uses the general-purpose tools from numerical algebra and the other is based on MCMC sampling.

      Consider $\bK = \bC^{\top} \bC$ with $\bC$ of size $d \times n$, for some $d\ll n$ \citep{KuTa11}, but still too large for exact sampling using the method of \citet{HKPV06}, then \citet{GiKuTa12} show how projecting $\bC$ can give an approximation with bounded error. 
      When this decomposition of the kernel is not possible, \citet{AKFT13} adapt \nystrom sampling \citep{WiSe01} to DPPs
      and bound the approximation error for DPPs and $k$-DPPs, which thus applies to projection DPPs.

      Apart from general purpose approximate solvers, there exist MCMC-based methods for approximate sampling from projection DPPs. 
      In Section~\ref{s:lyons}, we introduced the \textit{basis-exchange} property, which implies that once we remove an element from a basis $B_1$ of a linear matroid, any other basis $B_2$ has an element we can take and add to $B_1$ to make it a basis again. 
      This means we can construct a connected graph $G_{\text{be}}$ with $\cB$ as vertex set, and we add an edge between two bases if their symmetric difference has cardinality $2$. 
      $G_{\text{be}}$ is called the \textit{basis-exchange graph}. \citet{FeMi92} show that the simple random walk on $G_{\text{be}}$ has limiting distribution the uniform distribution on $\cB$ and mixes fast, under conditions that are satisfied by the matroids involved by DPPs. 
    
      If the uniform distribution on $\cB$ is not the DPP we want to sample from,\footnote{It may not even \emph{be} a DPP \citep[Corollary 5.5]{Lyo03}.} we can add an accept-reject step after each move to make the desired DPP the limiting distribution of the walk. 
      Adding such an acceptance step and a probability to stay at the current basis, \citet{AnGhRe16,LiJeSr16} give precise polynomial bounds on the mixing time of the resulting Markov chains. 
      This Markov kernel on $\cB$ is given in Algorithm~\ref{alg:basisExchange}. 
      Note that we use the acceptance ratio of \citet{LiJeSr16}.
      In the following, we make use of the notation $\Vol$ defined as follows.
      For any $P\subset [n]$,
      \begin{equation}
        \Vol^2(\bA_{:P}) \triangleq \det {\bA^{\top}}_{P:} \bA_{:P} \propto \det \bK_P,
      \end{equation}
      which corresponds to the squared volume of the parallelotope spanned by the columns of $\bA$ indexed by $P$.
      In particular, for subsets $P$ such that $|P|>r$ or such that $|P|=r$, $P\notin \cB$ we have $\Vol^2(\bA_{:P})=0$.
      However, for $B \in \cB$, $\Vol^2(\bB) = |\det \bA_{:B}|^2 > 0$. 
       
      We now turn to our contribution, which finds its place in this category of MCMC-based approximate
      DPP samplers.

      \begin{algorithm}[tb]
          \caption{\texttt{basisExchangeSampler}}
          \label{alg:basisExchange}
          \begin{algorithmic}
             \STATE {\bfseries Input:} Either $\bA$ or $\mathbf{K}$
             \STATE Initialize $i \leftarrow 0$ and pick $B_0 \in \cB$ as defined in \eqref{e:collection_of_bases-cB} 
             \WHILE{Not converged}
              \STATE Draw $ u \sim \Unif_{[0,1]}$
                  \IF{$u < \frac12$} 
                      \STATE Draw $s \sim \Unif_{B_i}$ and $t \sim \Unif_{[n]\setminus B_i}$
                      \STATE $P \leftarrow \lrp{B_{i}\setminus \{s\}} \cup \{t\}$
                      \STATE Draw $ u' \sim \Unif_{[0,1]}$
                          \IF{$u' < \frac{\Vol^2(\bA_{:P})}{\Vol^2(\bB_i) + \Vol^2(\bA_{:P})}
                              = \frac{\det \bK_{P}}{\det \bK_{B_i} + \det \bK_{P}}$}
                              \STATE $B_{i+1} \leftarrow P$
                          \ELSE
                              \STATE $B_{i+1} \leftarrow B_{i}$
                          \ENDIF
                  \ELSE
                      \STATE $B_{i+1} \leftarrow B_{i}$
                  \ENDIF
               \STATE $i\leftarrow i+1$
             \ENDWHILE
          \end{algorithmic}
      \end{algorithm}

  \section{Hit-and-run on Zonotopes}
  \label{s:zonotopes}

    Our main contribution is the construction of a fast-mixing Markov chain with limiting
    distribution a given projection DPP. Importantly, we assume to know $\bA$ in \eqref{e:projectionKernel}.
    \begin{assumption}
    \label{a:range}
      We know a full-rank $r\times n$ matrix $\bA$ such that $\bK = \bA^{\top}(\bA \bA^\top )^{-1} \bA$.
    \end{assumption}
    As discussed in Section~\ref{s:kdpps}, this is not an overly restrictive assumption, as many
    ML applications start with building the feature matrix $\bA$ rather than the similarity matrix $\bK$. 

    \subsection{Zonotopes}
    \label{ss:zonotopes}

      We define the \emph{zonotope} $\Zon(\bA)$ of $\bA$ as the $r$-dimensional volume spanned by the column vectors of $\bA$,
      \begin{equation}
      \label{e:Def_Zonotope}
        \Zon(\bA) = \bA [0,1]^n. 
      \end{equation}
      As an affine transformation of the unit hypercube, $\Zon(\bA)$ is a $r$-dimensional polytope.
      In particular, for a basis $B\in\cB$ of the matroid $M[\bA]$, the corresponding $\Zon(\bB)$ is a $r$-dimensional parallelotope with volume $\Vol(\bB) = \lrabs{\det\bB}$, see Figure~\ref{f:zonotope}.
      On the contrary, any $P\subset [n]$, such that $|P|=r, P\notin\cB$ also yields a parallelotope $\Zon(\bA_{:P})$, but its volume is null. 
      In the latter case, the exchange move in Algorithm~\ref{alg:basisExchange} will never be accepted and the state space of the corresponding Markov chain is indeed $\cB$.

      \setlength{\fboxsep}{0pt}
    	\begin{figure*}[!ht]
    		\subfigure[]{
    			\includegraphics[scale=0.21]{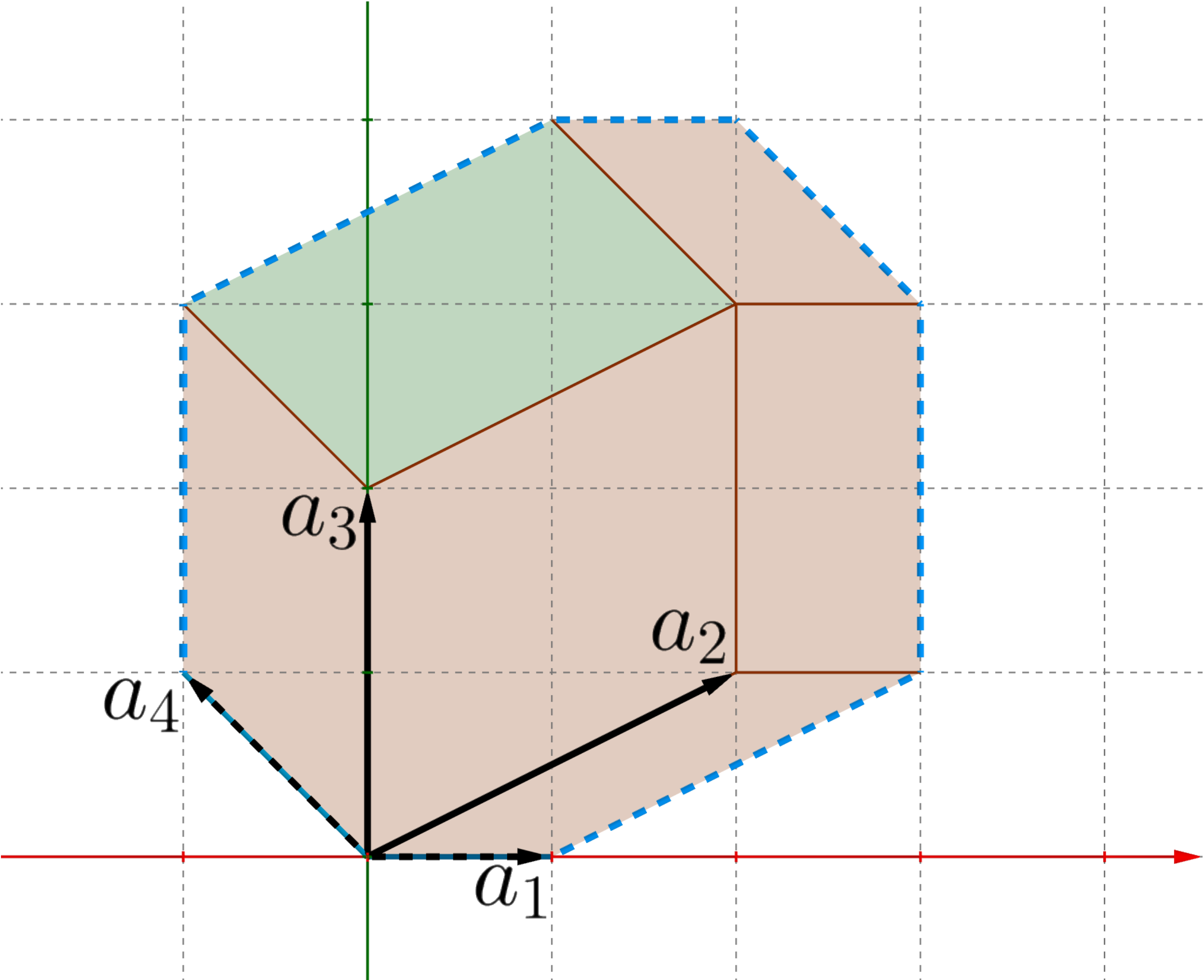}
    			\label{f:zonotope}
    			}
    		\subfigure[]{
    			\includegraphics[scale=0.24]{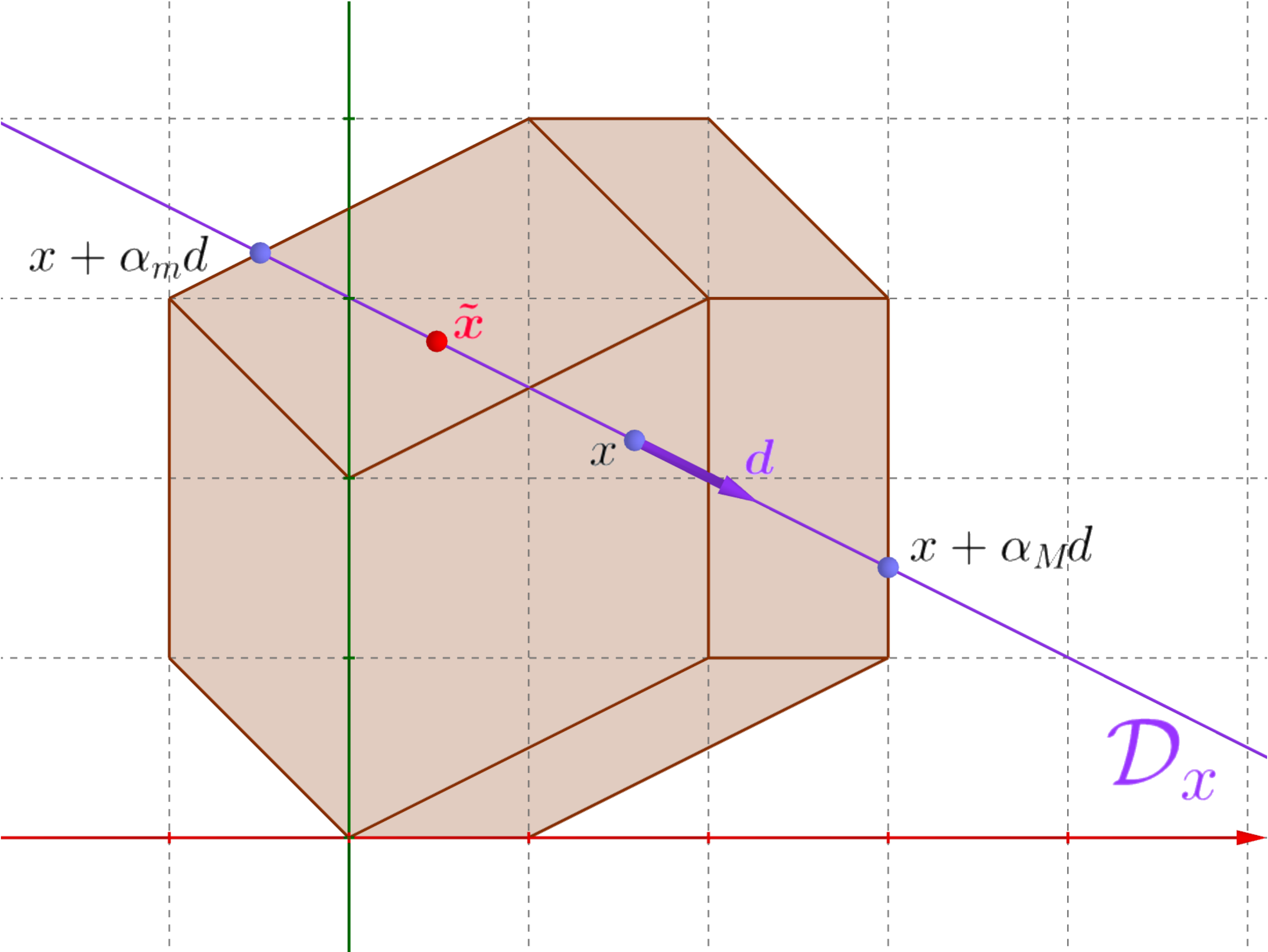}
    			\label{f:hitAndRun}
    			}
    		\subfigure[]{
    			\includegraphics[scale=0.21]{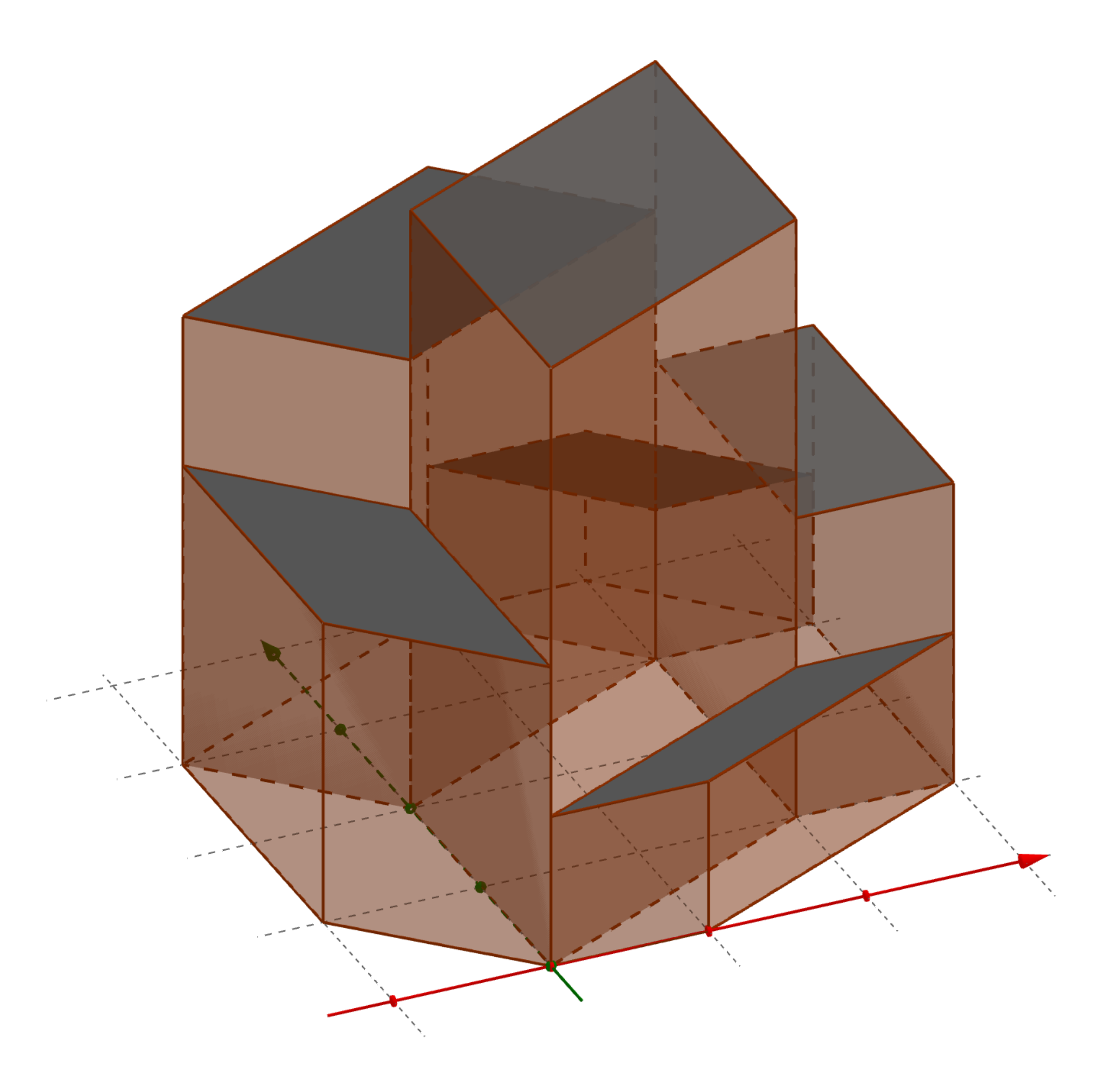}
    			\label{f:skyscrapers}
    			}
    		\caption{
    		(a) The dashed blue lines define the contour of $\Zon(\bA)$ where
    			$\bA=\lrp{
    			\begin{smallmatrix}
    			1 &2 &0 &-1\\
    			0 &1 &2 &1
    			\end{smallmatrix}}$. 
    			Each pair of column vectors corresponds to a parallelogram, the green one is associated to $\Zon(\mathbf{B})$ with $B=\lrcb{2,4}$. 
    		(b) A step of hit-and-run on the same zonotope. 
    		(c) Representation of $\pi_v$ for the same zonotope.}
    	\end{figure*}

      Our algorithm relies on the proof of the following.
      \begin{proposition}[see \citealp{DyFr94} for details]
      \label{p:dyer}
        \begin{equation}
        \label{e:Vol_Zonotope}
          \Vol(\Zon(\bA)) 
            = \sum_{B \in \cB} \Vol(\bB)
            = \sum_{B \in \cB} \lrabs{\det\bB}
        \end{equation}
      \end{proposition}

      \begin{proof}
        In short, for a good choice of $c\in\R^n$, \citet{DyFr94}
        consider for any $x \in \Zon(\bA)$, the following linear program (LP) noted $P_x(\bA,c)$,
        \begin{equation}
        \label{e:LP_Px}
            \begin{array}{cl}
              \min\limits_{y\in \R^n} & c^{\top} y        \\
                          \text{s.t.} & \bA y = x         \\
                                      & 0 \leq y \leq 1.
            \end{array}
        \end{equation}
        Standard LP results \citep{LuYe08} yield that the unique optimal solution $y^*$ of $P_x(\bA,c)$ takes the form
        \begin{equation}
        \label{e:sol_Px_Decomposition}
          y^* = \bA\xi(x) + \bB_x u,
        \end{equation}
        with $u \in [0,1]^r$ and $\xi(x) \in \{0,1\}^n$ such that $\xi(x)_i = 0$ for $i\in B_x$. 
        In case the choice of $B_x$ is ambiguous, \citet{DyFr94} take the smallest in the lexicographic order. 
        Decomposition~\eqref{e:sol_Px_Decomposition} allows locating any point $x\in \Zon(\bA)$ as falling inside a uniquely defined parallelotope $\Zon(\bB_x)$ shifted by $\xi(x)$. 
        Manipulating the optimality conditions of \eqref{e:LP_Px}, \citet{DyFr94} prove that each basis $B$ can be realized as a $B_x$ for some $x$, and that $x'\in\Zon(\bB_x)\Rightarrow B_x=B_{x'}$. 
        This allows to write $\Zon(\bA)$ as the tiling of all $\Zon(\bB)$, $B\in\cB$, with disjoint interiors. 
        This leads to Proposition~\ref{p:dyer}. 
      \end{proof}
      Note that $c$ is used to fix the tiling of the zonotope, but the map $x\mapsto B_x$ depends on this linear objective. Therefore, the tiling of $\Zon(\bA)$ is may not be unique. 
      An arbitrary~$c$ gives a valid tiling, as long as there are no ties when solving \eqref{e:LP_Px}. 
      \citet{DyFr94} use a nonlinear mathematical trick to fix $c$. In practice (Section \ref{s:USTs}), we generate a random Gaussian $c$ once and for all, which makes sure no ties appear during the execution, with probability 1.

      \begin{remark}
      \label{r:volumeSampling} 
        We propose to interpret the proof of Proposition~\ref{p:dyer} as a volume sampling algorithm: 
        if one manages to sample an $x$ uniformly on $\Zon(\bA)$, and then extracts the corresponding basis $B=B_x$ by
        solving \eqref{e:LP_Px}, then $B$ is drawn with probability proportional to $\Vol(\bB) = |\det\bB|$.
      \end{remark} 

      Remark~\ref{r:volumeSampling} is close to what we want, as sampling from a projection DPP under Assumption~\ref{a:range} boils down to sampling a basis $B$ of $M[\bA]$  proportionally to the \emph{squared} volume $|\det\bB|^2$  (Section~\ref{s:lyons}). 
      In the rest of this section, we explain how to efficiently sample $x$ uniformly on $\Zon(\bA)$, and how to change the volume into its square.

    \subsection{Hit-and-run and the Simplex Algorithm}
    \label{s:hitAndRun}

      $\Zon(\bA)$ is a convex set. 
      Approximate uniform sampling on large-dimensional convex bodies is one of the core questions in MCMC, see e.g., \citet{CoVe16} and references therein. 
      The hit-and-run Markov chain \citep{Tur71,Smi84} is one of the preferred practical and theoretical solutions \citep{CoVe16}.

      We describe the Markov kernel $P(x,z)$ of the hit-and-run Markov chain for a generic target distribution $\pi$ supported on a convex set $C$. 
      Sample a point $y$ uniformly on the unit sphere centered at $x$. 
      Letting $d=y-x$, this defines the line $\cD_x\triangleq \lrcb{x + \alpha d ~;~ \alpha\in \R}$. 
      Then, sample $z$ from any Markov  kernel $Q(x,\cdot)$ supported on $\cD_x$ that leaves the restriction of $\pi$ to $\cD_x$ invariant. 
      In particular, Metropolis-Hastings kernel (MH, \citealt{RoCa04}) is often used with uniform proposal on $\cD_x$, which favors large moves across the support $C$ of the target, see Figure~\ref{f:hitAndRun}. 
      The resulting Markov kernel leaves $\pi$ invariant, see e.g., \citet{AnDi07} for a general proof. 
      Furthermore, the hit-and-run Markov chain has polynomial mixing time for log concave $\pi$ \citep[Theorem 2.1]{LoVe03}. 

      To implement Remark~\ref{r:volumeSampling}, we need to sample from $\pi_{u}\propto \indic_{\Zon(\bA)}$.
      In practice, we can choose the secondary Markov kernel $Q(x,\cdot)$ to be MH with uniform proposal  on~$\cD_x$, as long as we can determine the endpoints $x + \alpha_m (y-x)$ and $x + \alpha_M (y-x)$ of $\cD_x\cap\Zon(\bA)$.
      In fact, zonotopes are tricky convex sets, as even an oracle saying whether a point belongs to the zonotope requires solving LPs (basically, it is Phase I of the simplex algorithm).
      As noted by~\citet[Section 4.4]{LoVe03}, hit-and-run with LP is the state-of-the-art for computing the volume of large-scale zonotopes. 
      Thus, by definition of $\Zon(\bA)$, this amounts to solving two more LPs:
      $\alpha_m$ is the optimal solution to the linear program
      \begin{equation}
      \label{e:LP_alpha}
        \begin{array}{cl}
          \min \limits_{\lambda \in \R^n, \alpha\in \R} & \alpha          \\
                            \text{s.t.} & x + \alpha d = \bA \lambda  \\
                                  & 0 \leq \lambda \leq 1,
        \end{array}
      \end{equation}
      while $\alpha_M$ is the optimal solution of the same linear program with objective $-\alpha$. 
      Thus, a combination of hit-and-run and LP solvers such as Dantzig's simplex algorithm \citep{LuYe08} yields a Markov kernel with invariant distribution $\indic_{\Zon(\bA)}$, summarized in Algorithm~\ref{alg:unifZonoHitRun}.
      The acceptance in MH is $1$ due to our choice of the proposal and the target. 
      By the proof of Proposition~\ref{p:dyer}, running Algorithm~\ref{alg:unifZonoHitRun}, taking the output chain $(x_i)$ and extracting the bases $(B_{x_i})$ with Algorithm~\ref{alg:extractBasis}, we obtain a chain on $\cB$ with invariant distribution proportional to the volume of $\bB$.

      \begin{algorithm}[tb]
        \caption{\texttt{unifZonoHitAndRun}}
        \label{alg:unifZonoHitRun}
        \begin{algorithmic}
          \STATE {\bfseries Input:} $\bA$
          \STATE {\bfseries Initialization:} 
          \STATE $i \leftarrow 0$
          \STATE $x_0 \leftarrow \bA u$ with $u\sim \Unif_{[0,1]^n}$
          \WHILE{Not converged}
            \STATE Draw $d\sim\Unif_{\mathbb{S}^{r-1}}$ and let $\cD_{x_{i}} \triangleq x_i + \R d$
            \STATE Draw $\tilde{x} \sim \Unif_{\cD_{x_i} \cap \Zon(A)}$
            \quad\textcolor{magenta}{\#Solve 2 LPs, see \eqref{e:LP_alpha} }
            \STATE $x_{i+1} \leftarrow \tilde{x}$
            \STATE $i\leftarrow i+1$
          \ENDWHILE
        \end{algorithmic}
      \end{algorithm}

      \begin{algorithm}[tb]
        \caption{\texttt{extractBasis}}
        \label{alg:extractBasis}
        \begin{algorithmic}
          \STATE {\bfseries Input:} $\bA, c, x\in \Zon(\bA)$
          \STATE Compute $y^*$ the opt.\,solution of $P_x(\bA,c)$
              \ \textcolor{magenta}{\#1 LP, see \eqref{e:LP_Px}}
          \STATE $B \leftarrow \lrcb{i ~; y^*_i \in ]0,1[}$
          \STATE \textbf{return} $B$
        \end{algorithmic}
      \end{algorithm}

      In terms of theoretical performance, this Markov chain inherits \citepos{LoVe03} mixing time as it is a simple transformation of hit-and-run targeting the uniform distribution on a convex set.
      We underline that this is not a pathological case and it already covers a range of applications, as changing the feature matrix $\bA$ yields another zonotope, but the target distribution on the zonotope stays uniform. 
      Machine learning practitioners do not use volume sampling for diversity sampling yet, but nothing prevents it, as it already encodes the same feature-based diversity as squared volume sampling (i.e., DPPs).
      Nevertheless, our initial goal was to sample from a projection DPP with kernel $\bK$ under Assumption~\ref{a:range}. 
      We now modify the Markov chain just constructed to achieve that.

    \subsection{From Volume to Squared Volume}
    \label{s:skyscrapers}

      Consider the probability density function on $\Zon(\bA)$
      \begin{equation*}
        \pi_v(x) = \frac{\lrabs{\det \bB_x}}{\det \bA\bA^{\top}} \indic_{\Zon(\bA)}(x),
      \end{equation*}
      represented on our example in Figure~\ref{f:skyscrapers}. 
      Observe, in particular, that $\pi_v$ is constant on each $\Zon{(\bB)}$.
      Running the hit-and-run algorithm with this target instead of $\pi_u$ in Section~\ref{s:hitAndRun}, and extracting bases using Algorithm~\ref{alg:extractBasis} again, we obtain a Markov chain on $\cB$ with limiting distribution $\nu(B)$ proportional to the squared volume spanned by column vectors of $\bB$, as required. 
      To see this, note that $\nu(B)$ is the volume of the ``skyscraper'' built on top of $\Zon(\bB)$ in Figure~\ref{f:skyscrapers}, that is $\Vol(\bB)\times \Vol(\bB)$.

      The resulting algorithm is shown in Algorithm~\ref{alg:zonoSampling}.
      Note the acceptance ratio in the subroutine Algorithm~\ref{alg:volZonoHitRun} compared to Algorithm~\ref{alg:unifZonoHitRun}, since the target of the hit-and-run algorithm is not uniform anymore.

      \begin{algorithm}[tb]
        \caption{\texttt{volZonoHitAndRun}}
        \label{alg:volZonoHitRun}
        \begin{algorithmic}
          \STATE {\bfseries Input:} $ \bA, c, x, B$
          \STATE Draw $d\sim\Unif_{\mathbb{S}^{r-1}}$ and let $\cD_{x} \triangleq x + \R d$
          \STATE Draw $\tilde{x} \sim \Unif_{\cD_x \cap \Zon(\bA)}$ 
            \quad\textcolor{magenta}{\#Solve 2 LPs, see \eqref{e:LP_alpha}}
          \STATE $\tilde{B} \leftarrow $ \texttt{extractBasis($\bA, c, \tilde{x}$)}
            \quad\textcolor{magenta}{\#Solve 1 LP, see \eqref{e:LP_Px}}
          \STATE Draw $u\sim \Unif_{[0,1]}$
              \IF{
              $u < \frac{\Vol(\tilde{\bB})}{\Vol(\bB)} = \lrabs{\frac{\det \bA_{:\tilde{B}}}{\det \bA_{:B}}}$
              }  
                \STATE \textbf{return} $\tilde{x}, \tilde{B}$
            \ELSE
              \STATE \textbf{return} $x, B$
              \ENDIF
        \end{algorithmic}
            \end{algorithm}
      \begin{algorithm}[tb]
        \caption{\texttt{zonotopeSampler}}
        \label{alg:zonoSampling}
        \begin{algorithmic}
          \STATE {\bfseries Input:} $\bA, c$
          \STATE {\bfseries Initialization:}
          \STATE $i \leftarrow 0$
          \STATE $x_i \leftarrow \bA u$, with $u\sim \Unif_{[0,1]^n}$
          \STATE $B_i \leftarrow $ \texttt{extractBasis($\bA, c, x_i$)}
          \WHILE{Not converged}
            \STATE $x_{i+1}, B_{i+1} \leftarrow $ \texttt{volZonoHitAndRun($\bA, c, x_i, B_i$)}
            \STATE $i\leftarrow i+1$
          \ENDWHILE
        \end{algorithmic}
      \end{algorithm}

    \subsection{On Base Measures}
    \label{s:baseMeasure}
      As described in Section~\ref{s:kdpps}, it is common in ML to specify a marginal relevance $q_i$ of each item $i\in [n]$, i.e., the \emph{base measure} of the DPP. 
      Compared to a uniform base measure, this means replacing $\bA$ by $\tilde{\bA}$ with columns $\tilde{a_i}=\sqrt{q_i}a_i$. 
      Contrary to $\bA$, in Algorithm~\ref{alg:volZonoHitRun}, both the zonotope and the acceptance ratio are scaled by the corresponding products of $\sqrt{q_i}$s.
      We could equally well define $\tilde{\bA}$ by multiplying each column of $\bA$ by $q_i$ instead of its square root, and leave the acceptance ratio in Algorithm~\ref{alg:volZonoHitRun} use columns of the original $\bA$. By the arguments in Section~\ref{s:skyscrapers}, the chain $(B_i)$ would leave the same projection DPP invariant.
      In particular, we have some freedom in how to introduce the marginal relevance $q_i$, so we can choose the latter solution that simply scales the zonotope and its tiles to preserve outer angles, while using unscaled volumes to decide acceptance.
      This way, we do not create harder-to-escape or sharper corners for hit-and-run, which could lead the algorithm to be stuck for a while \citep[Section 4.2.1]{CoVe16}.
      Finally, since hit-and-run is efficient at moving across convex bodies \citep{LoVe03}, the rationale is that if hit-and-run was empirically mixing fast before scaling, its performance should not decrease.

  \section{Experiments}
  \label{s:experiments}

    We investigate the behavior of our Algorithm \ref{alg:zonoSampling} on
    synthetic graphs in Section~\ref{s:USTs}, in summary extraction in Section~\ref{s:text_summarization}, and on MNIST in Appendix~\ref{s:app_mnist}.
  \vspace{-0.5em}
    \subsection{Non-uniform Spanning Trees}
    \label{s:USTs}
      We compare Algorithm~\ref{alg:basisExchange} studied by \citet{AnGhRe16,LiJeSr16} and our Algorithm~\ref{alg:zonoSampling} on two types of graphs, in two different settings.
      The graphs we consider are the complete graph $K_{10}$ with 10 vertices (and 45 edges) and a realization $\text{BA}(20,2)$ of a Barab\'{a}si-Albert graph with 20 vertices and parameter 2. 
      We chose BA as an example of structured graph, as it has the preferential attachment property present in social networks \citep{BaAl99}. 
      The input matrix $\bA$ is a weighted version of the vertex-edge incidence matrix of each graph for which we keep only the~9 (resp.\,19) first rows, so that it satisfies Assumption~\ref{a:range}. 
      For more generality, we introduce a base measure, as described in Section~\ref{s:kdpps} and \ref{s:baseMeasure}, by reweighting the columns  of $\bA$ with i.i.d.\,uniform variables in $[0,1]$. 
      Samples from the corresponding projection DPP are thus spanning trees drawn proportionally to the products of their edge weights.

      For Algorithm~\ref{alg:zonoSampling}, a value of the linear objective $c$ is drawn once and for all, for each graph, from a standard Gaussian distribution.
      This is enough to make sure no ties appear during the execution, as mentioned in Section \ref{ss:zonotopes}. 
      This linear objective is kept fixed throughout the experiments so that the tiling of the zonotope remains
      the same.
      We run both algorithms for 70 seconds, which corresponds to roughly 50\,000 iterations of Algorithm~\ref{alg:zonoSampling}.
      Moreover, we run 100 chains in parallel for each of the two algorithms. 
      For each of the 100 repetitions, we initialize the two algorithms with the same random initial basis, obtained by solving \eqref{e:LP_Px} once, with $x=\bA u$ and $u\sim \Unif_{[0,1]^n}$.
      For both graphs, the total number $|\cB|$ of bases is of order $10^8$, so computing total variation distances is impractical. 
      We instead compare Algorithms~\ref{alg:basisExchange} and \ref{alg:zonoSampling} based on the estimation of inclusion probabilities $\Proba{S\subset B}$ for various subsets $S\subset [n]$ of size $3$. 
      We observed similar behaviors across 3-subsets, so we display here the typical behavior on a 3-subset.

      The inclusion probabilities are estimated via a running average of the number of bases containing the subsets $S$. 
      Figures~\ref{f:xp_graph_Kn_weighted_iter_triplet} and \ref{f:xp_graph_BA_weighted_iter_triplet} show the behavior of both algorithms vs.\,MCMC iterations for the complete graph $K_{10}$ and a realization of $\text{BA}(20,2)$, respectively.
      Figures~\ref{f:xp_graph_Kn_weighted_CPU_triplet} and \ref{f:xp_graph_BA_weighted_CPU_triplet} show the behavior of both algorithms vs.\,wall-clock time for the complete graph $K_{10}$ and a realization of $\text{BA}(20,2)$, respectively.
      In these four figures, bold curves correspond to the median of the relative errors, whereas the frontiers of colored regions indicate the first and last deciles of the relative errors.

      In Figures~\ref{f:xp_graph_Kn_weighted_PSRF_triplet} and \ref{f:xp_graph_BA_weighted_PSRF_triplet} we compute the Gelman-Rubin statistic \citep{GeRu92}, also called the potential scale reduction factor (PSRF). 
      We use the PSRF implementation of CODA \citep{codaRpackage} in R, on the 100 binary chains indicating the presence of the typical 3-subset in the current basis.

      \begin{figure*}[!ht]
      \vspace{-1em}
      \centering
          \subfigure[Relative error vs.\,MCMC iterations.]{
          \includegraphics[width=\threefig]{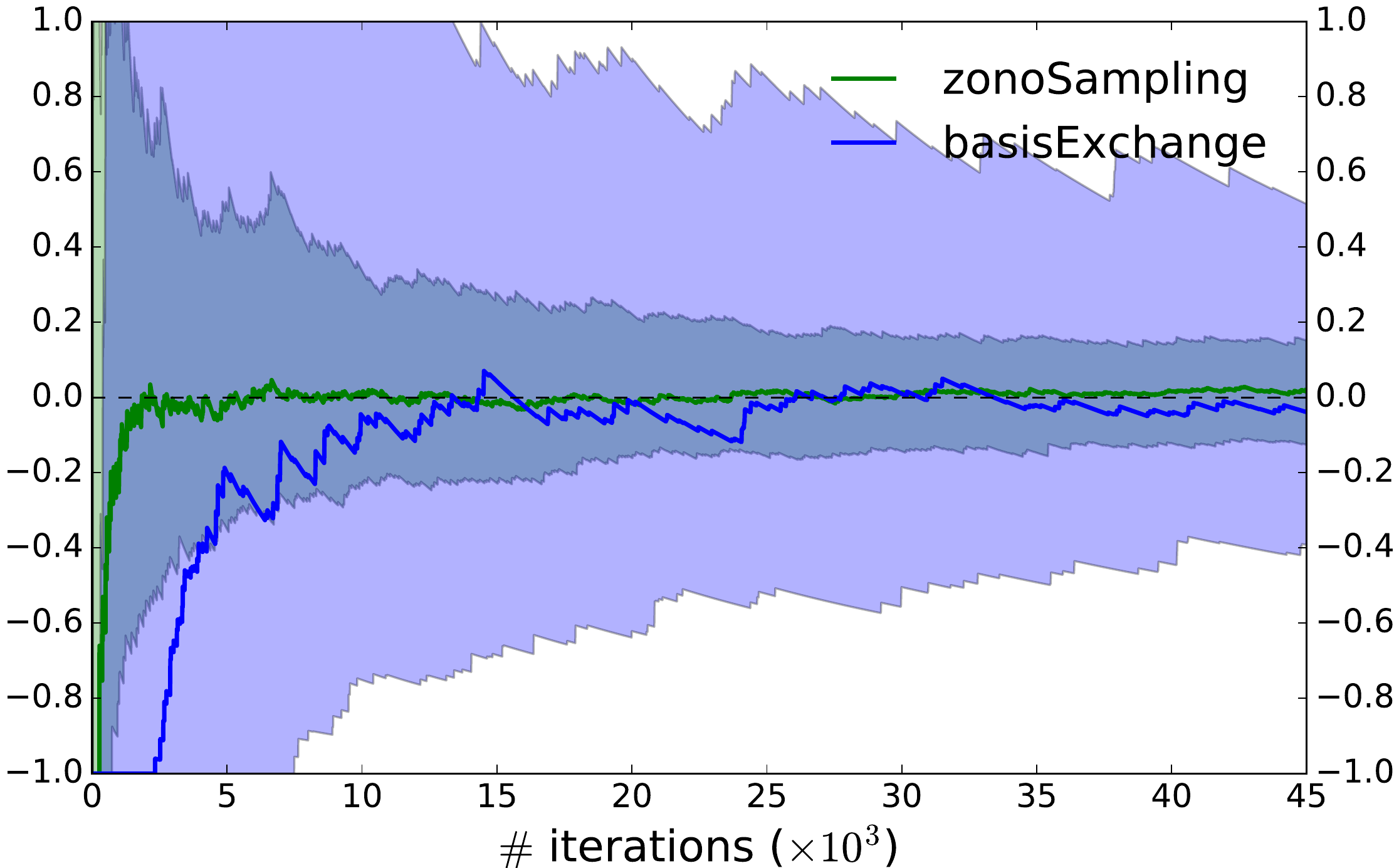}
          \label{f:xp_graph_Kn_weighted_iter_triplet}
          }
          \subfigure[Relative error vs.\,wall-clock time.]{
          \includegraphics[width=\threefig]{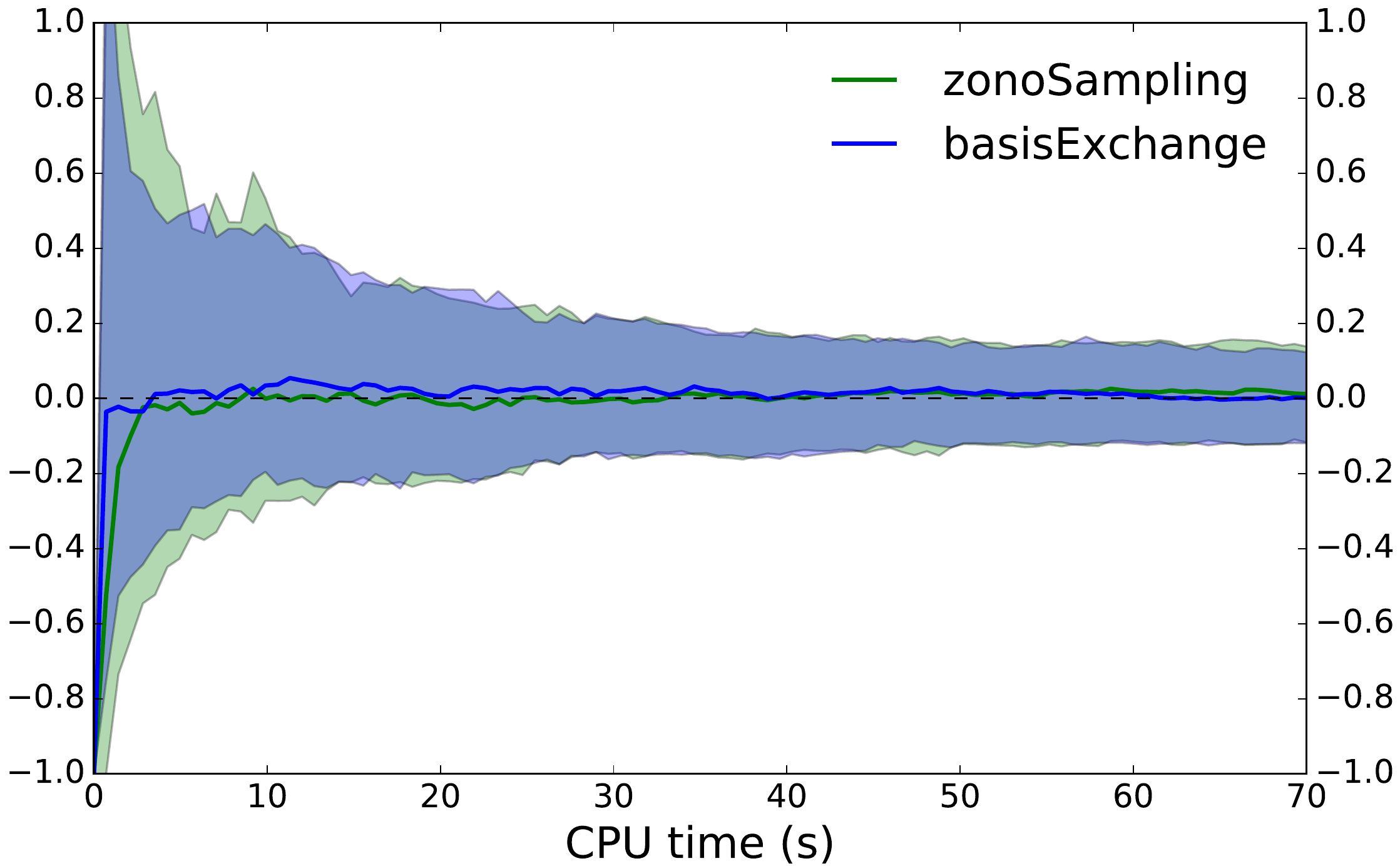}
          \label{f:xp_graph_Kn_weighted_CPU_triplet}
          }
          \subfigure[PSRF vs. MCMC iterations.]{
          \includegraphics[width=\threefig, height=88pt]{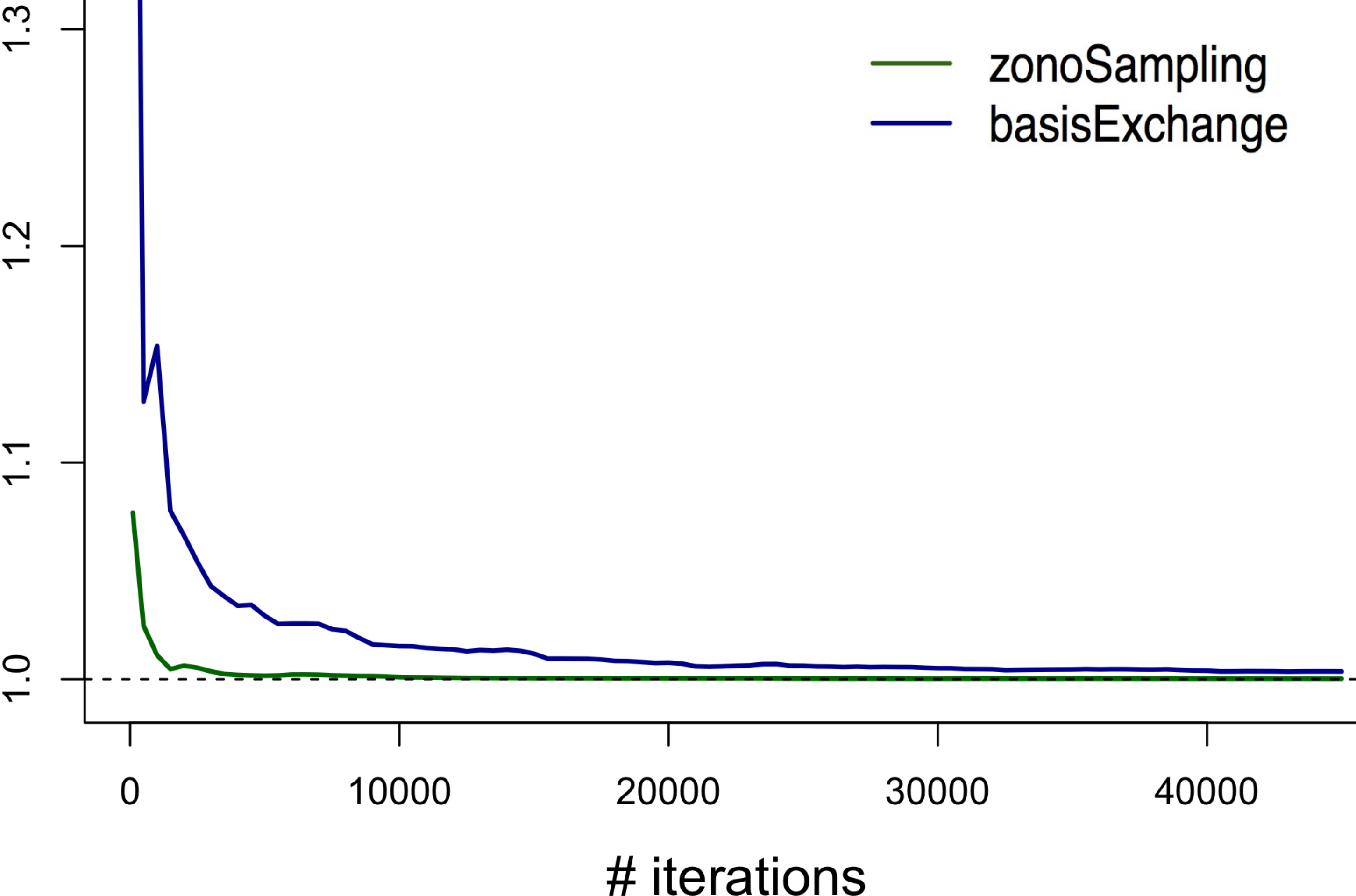}
          \label{f:xp_graph_Kn_weighted_PSRF_triplet}
          }
           \vspace{-1em}
          \caption{Comparison of Algorithms~\ref{alg:basisExchange} and~\ref{alg:zonoSampling} on the complete graph $K_{10}$.}
          \label{f:xp_graph_Kn_weighted_triplet}
      \end{figure*}

      \begin{figure*}[!ht]
          \vspace{-1em}
       \centering
          \subfigure[Relative error vs.\,MCMC iterations.]{
          \includegraphics[width=\threefig]{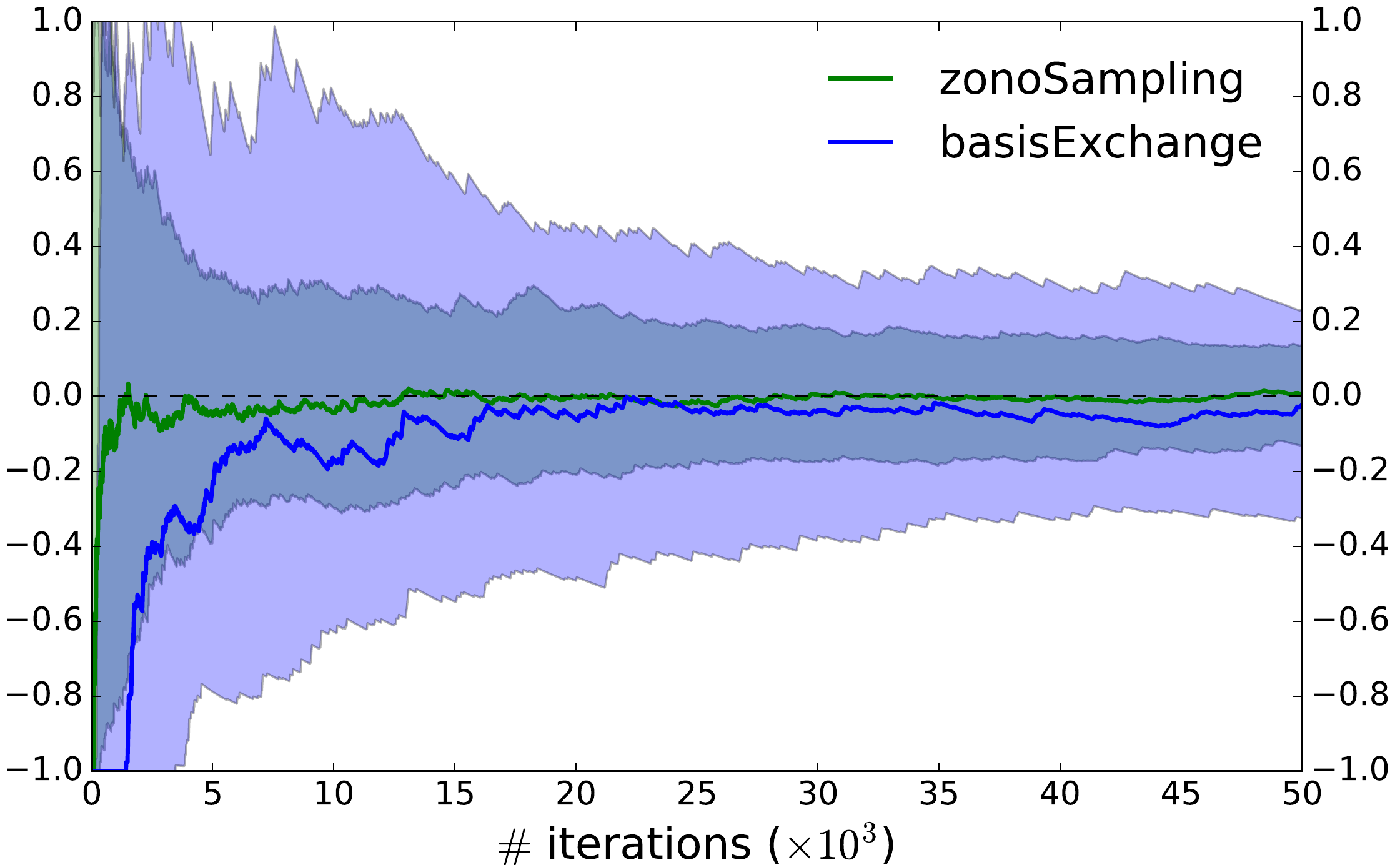}
          \label{f:xp_graph_BA_weighted_iter_triplet}
          }
          \subfigure[Relative error vs.\,wall-clock time.]{
          \includegraphics[width=\threefig]{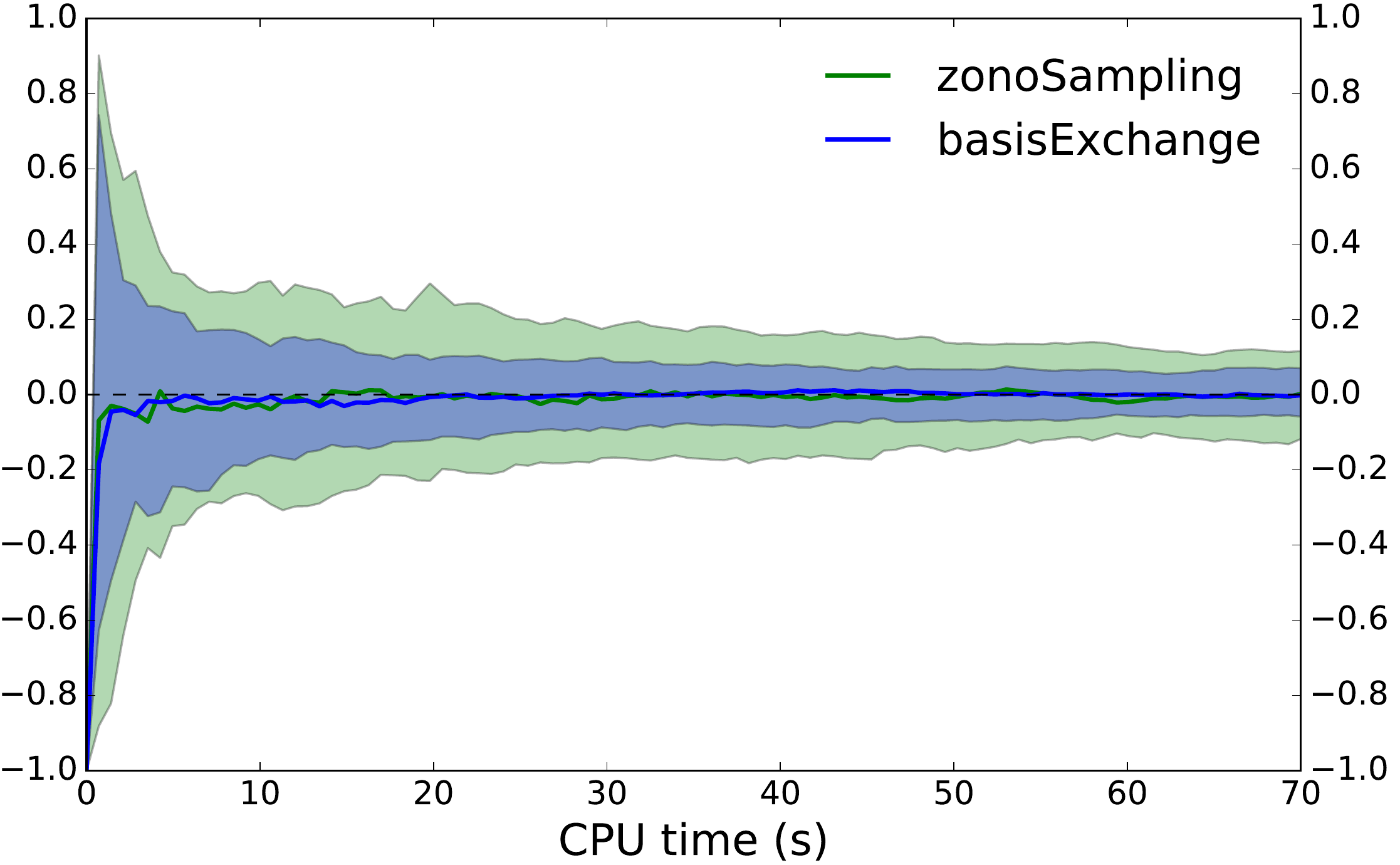}
          \label{f:xp_graph_BA_weighted_CPU_triplet}
          }
          \subfigure[PSRF vs. MCMC iterations.]{
          \includegraphics[width=\threefig, height=88pt]{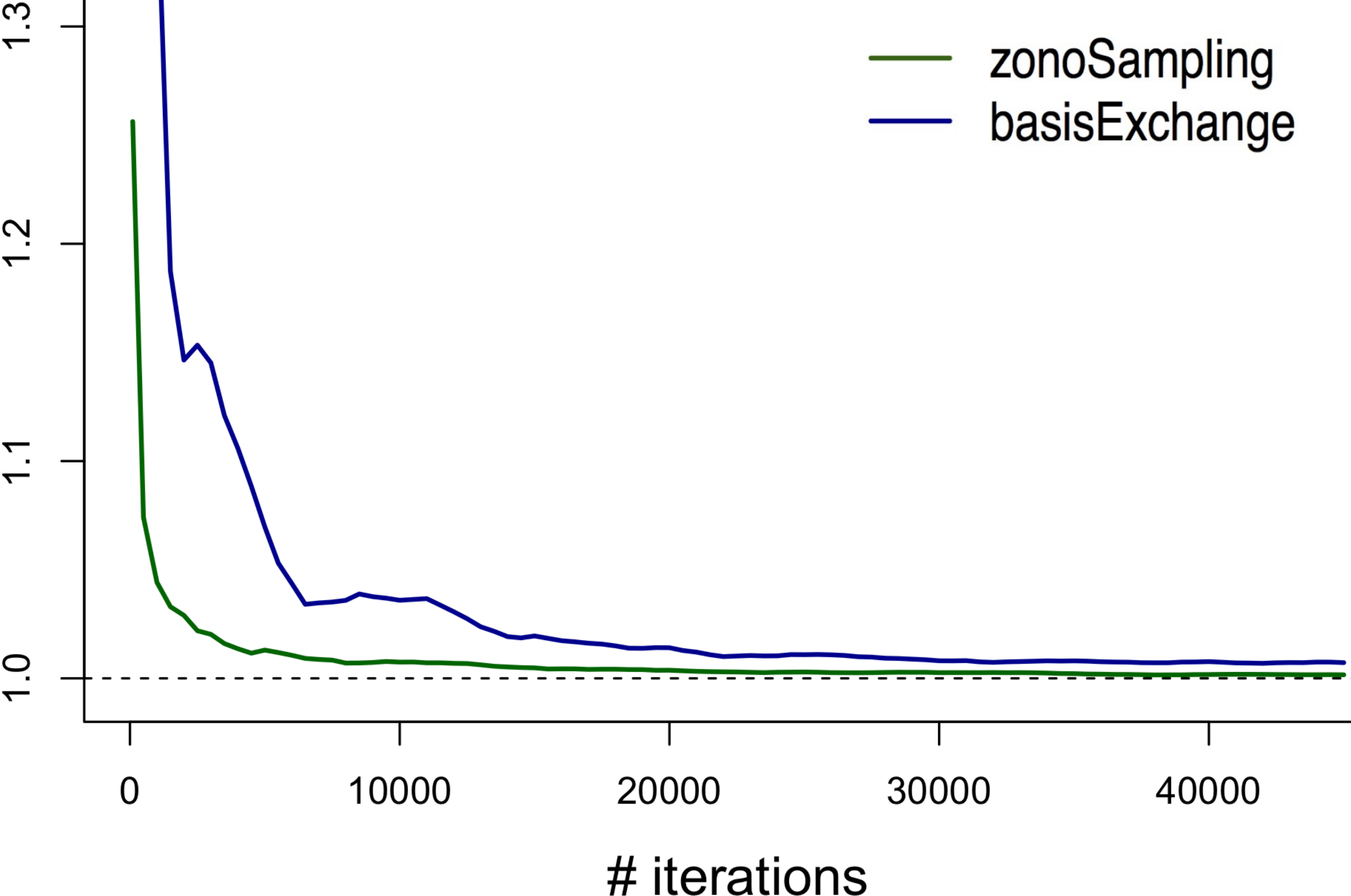}
          \label{f:xp_graph_BA_weighted_PSRF_triplet}
          }
           \vspace{-1em}
          \caption{Comparison of Algorithms~\ref{alg:basisExchange} and~\ref{alg:zonoSampling} on a realization of $\text{BA}(20,2)$.}
          \label{f:xp_graph_BA_weighted_triplet}
      \end{figure*} 

      In terms of number of iterations, our Algorithm~\ref{alg:zonoSampling} clearly mixes faster. 
      Relatedly, we observed typical acceptance rates for our algorithm an order of magnitude larger than Algorithm~\ref{alg:basisExchange}, while simultaneously attempting more global moves than the local basis-exchange moves of Algorithm~\ref{alg:basisExchange}.
      The high acceptance is partly due to the structure of the zonotope: the uniform proposal in the hit-and-run algorithm already favors bases with large determinants, as the length of the intersection of $D_x$ in Algorithm~\ref{alg:volZonoHitRun} with any $\Zon(\bB)$ is an indicator of its volume, see also Figure~\ref{f:hitAndRun}.

      Under the time-horizon constraint, see Figures~\ref{f:xp_graph_Kn_weighted_CPU_triplet} and \ref{f:xp_graph_BA_weighted_CPU_triplet}, Algorithm~\ref{alg:basisExchange} has time to perform more than $10^6$ iterations compared to roughly 50\,000 steps for our chain. 
      The acceptance rate of Algorithm~\ref{alg:zonoSampling} is still 10 times larger, but the time required to solve the
      linear programs at each MCMC iteration clearly hinders our algorithm in terms of CPU time.
      Both algorithms are comparable in performance, but given its large acceptance, we would expect our algorithm to perform better if it was allowed to do even only 10 times more iterations. 
      Now this is implementation-dependent, and our current implementation of Algorithm~\ref{alg:zonoSampling} is relatively naive, calling the simplex algorithm in the GLPK \citep{Oki12} solver with CVXOPT \citep{AnDaVa12} from Python. 
      We think there are big potential speed-ups to realize in the integration of linear programming solvers in our
      code.
      Moreover, we initialize our simplex algorithms randomly, while the different LPs we solve are related, so there may be additional smart mathematical speed-ups in using the path followed by one simplex instance to initialize the next.

      Finally, we note that the performance of our Algorithm~\ref{alg:zonoSampling} seems stable and independent of the structure of the graph, while the performance of the basis-exchange Algorithm~\ref{alg:basisExchange} seems more graph-dependent.
      Further investigation is needed to make stronger statements.

    \subsection{Text Summarization} % (fold)
    \label{s:text_summarization}
    
      Looking at Figures~\ref{f:xp_graph_Kn_weighted_triplet} and \ref{f:xp_graph_BA_weighted_triplet}, our algorithm will be most useful when the bottleneck is mixing vs.\,number of iterations rather than CPU time. 
      For instance, when integrating a costly-to-evaluate function against a projection DPP, the evaluation of the integrand may outweigh the cost of one iteration. 
      To illustrate this, we adapt an experiment of \citet[Section 4.2.1]{KuTa12} on minimum Bayes risk decoding for summary extraction. 
      The idea is to find a subset $Y$ of sentences of a text that maximizes
      \begin{equation}
      \label{e:rouge}
      \frac1R \sum_{r=1}^{R} \textsc{Rouge-1F} \lrp{Y, Y_r},
      \end{equation}
      where $\lrp{Y_r}_r$ are sampled from a projection DPP. \textsc{Rouge-1F} is a measure of similarity of two sets of sentences.
      We summarize this 64-sentence \href{http://www.slate.com/articles/health_and_science/science/2017/04/explaining_science_won_t_fix_information_illiteracy.html}{article} as a subset of $11$ sentences. 
      In this setting, evaluating once $\textsc{Rouge-1F}$ in the sum~\eqref{e:rouge} takes $0.1s$ on a modern laptop, while one iteration of our algorithm is $10^{-3}s$. 
      Our Algorithm~\ref{alg:zonoSampling} can thus compute \eqref{e:rouge} for $R=10\,000$ in about the same CPU time as Algorithm~\ref{alg:basisExchange}, an iteration of which costs $10^{-5}s$. 
      We show in Figure~\ref{f:xp_summary_box_plots} the value of \eqref{e:rouge} for 3 possible summaries $\lrp{Y^{(i)}}_{i=1}^3$ chosen uniformly at random in $\cB$, over $50$ independent runs.
      The variance of our estimates is smaller, and the number of different summaries explored is about $50\%$, against $10\%$ for Algorithm~\ref{alg:basisExchange}.
      Evaluating \eqref{e:rouge} using our algorithm is thus expected to be closer to the maximum of the underlying integral. Details are given in Appendix~\ref{s:app_summarization}.

  \section{Discussion}
  \label{s:discussion}
   
    We proposed a new MCMC kernel with limiting distribution being an arbitrary projection DPP. 
    This MCMC kernel leverages optimization algorithms to help making global moves on a convex body that represents the DPP. 
    We provided empirical results supporting its fast mixing when compared to the state-of-the-art basis-exchange chain of \citet{AnGhRe16, LiJeSr16}. 
    Future work will focus on an  implementation: while our MCMC chain mixes faster, when compared based on CPU time our algorithm suffers from having to solve linear programs at each iteration. 
    We note that even answering the question whether a given point belongs to a zonotope involves linear programming, so that chord-finding procedures used in slice sampling \citep[Sections~4 and~5]{Nea03} would not provide significant computational savings. 
   
    We also plan to investigate theoretical bounds on the mixing time of our Algorithm~\ref{alg:volZonoHitRun}. 
    We can build upon the work of \citet{AnGhRe16}, as our Algorithm~\ref{alg:volZonoHitRun} is also a weighted extension of our Algorithm~\ref{alg:unifZonoHitRun}, and the polynomial bounds for the vanilla hit-and-run algorithm \citep{LoVe03} already apply to the latter. 
    Note that while not targeting a DPP, our Algorithm~\ref{alg:unifZonoHitRun} already samples items with feature-based repulsion, and could be used independently if the determinantal aspect is not crucial to the application. 
    
    \begin{figure}[t]
    \vspace{-0.5em}
    \centering
      \includegraphics[width=\twofig]{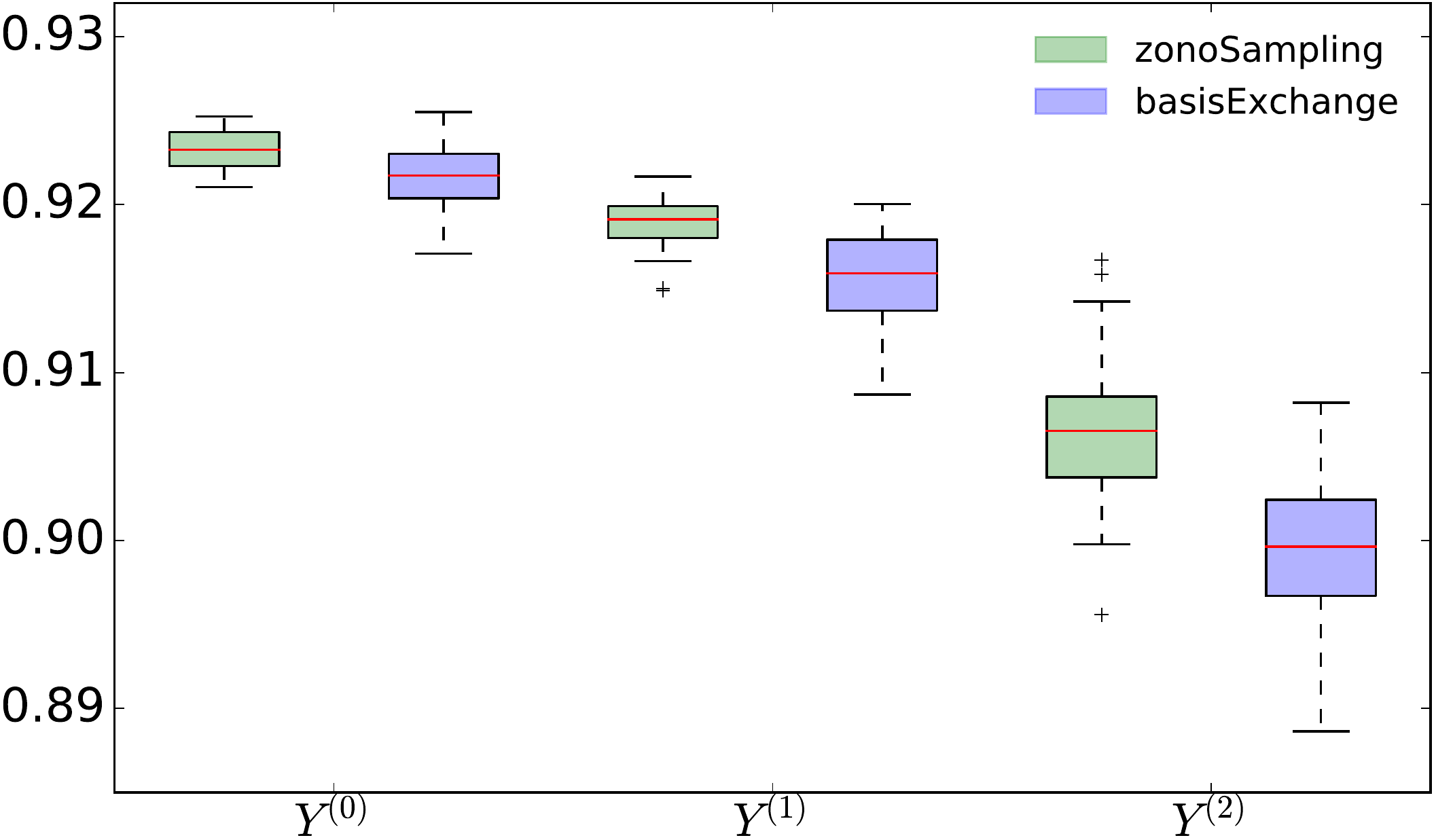}
    \vspace{-1em}
      \caption{Summary extraction results}
         \vspace{-1em}
        \label{f:xp_summary_box_plots}
  \end{figure} 
    
  {\small
  \textbf{Acknowledgments} The research presented was supported by French Ministry of
    Higher Education and Research, CPER Nord-Pas de Calais/FEDER DATA Advanced data science and technologies 2015-2020, and French National Research Agency projects \textsc{ExTra-Learn} (n.ANR-14-CE24-0010-01) and \textsc{BoB} (n.ANR-16-CE23-0003).}
  \balance
     \bibliographystyle{icml2017}
      \bibliography{biblio}

  \appendix

  \onecolumn{

    \section{Landmarks on MNIST}
    \label{s:app_mnist}

      To illustrate our Algorithm~\ref{alg:zonoSampling}{} on a non-synthetic dataset, we take a thousand images of handwritten $1s$ from MNIST. 
      We consider the problem of finding landmarks among these $1s$, see, for example, \citet{LiPa15}. 
      To obtain a set of $r=10$ landmarks from a projection DPP, we design $r$-dimensional feature vectors of our images, the inner products of which indicate similarity. 
      Specifically, we take $\bA$ to be the coordinates of the projections of our dataset onto the first~$r$ principal directions given by a standard PCA, so that $\bA$ has rank $r$ if the whole variance is not explained by these first~$r$ eigenvalues. 
      Note that our choice of using standard PCA is arbitrary: For manifold landmarking, we could take the coordinates output by a manifold learning algorithm such as ISOMAP. 

      Running our Algorithm~\ref{alg:zonoSampling} for $10\,000$ time steps, we obtain bases with squared volumes spanning three orders of magnitude. 
      The basis in our MCMC sample with maximum squared volume is shown in the first row of Figure~\ref{f:mnist}. 
      The bottom three rows show bases drawn uniformly at random from our initial dataset, for visual comparison. 
      The left column gives the log of the ratio of the squared volume of the corresponding basis by that of the top row.

      \begin{figure*}[!ht]
       \centering
         \includegraphics[width=.95\textwidth]{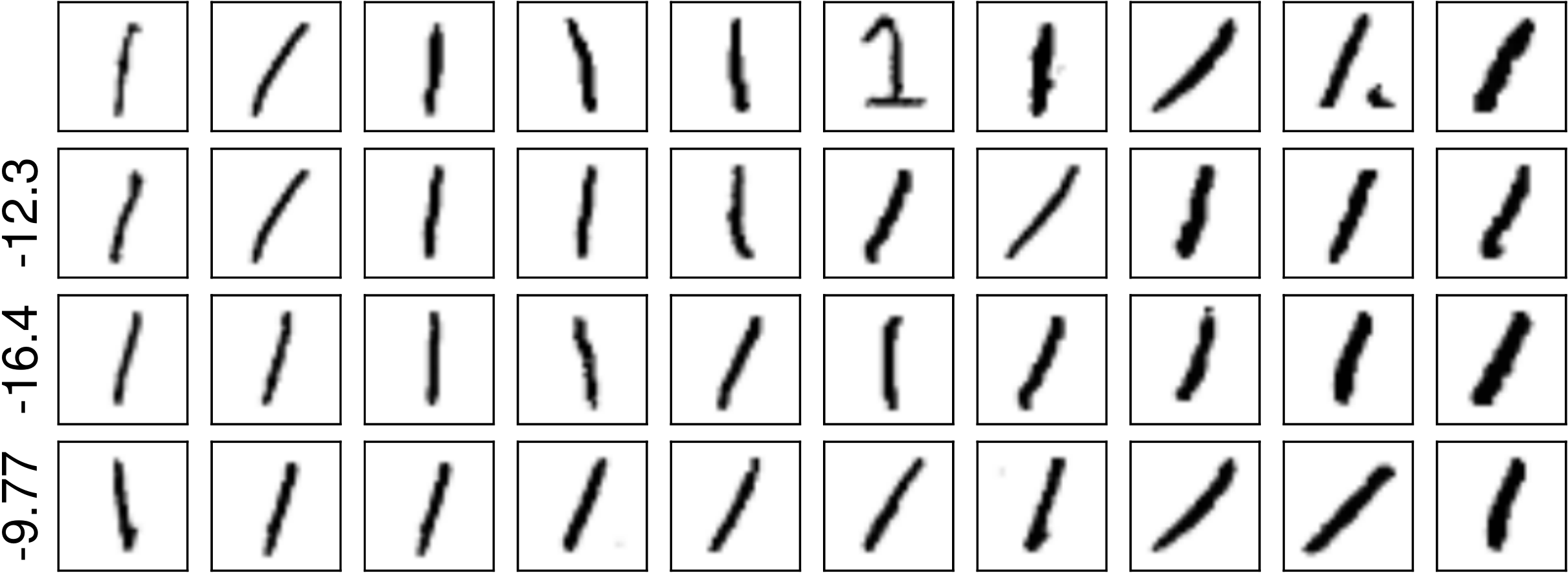}
         \caption{Results of the MNIST experiment in Section~\ref{s:app_mnist}: the MAP basis in the first row is compared to uniform random bases.}
       \label{f:mnist}
       \end{figure*}

      Visual inspection reveals that the projection DPP and our MCMC sample successfully pick up diversity among 1s: Different angles, thicknesses, and shapes are selected, while uniform random sampling exhibits less coverage. 
      This is confirmed by the log ratios, which are far from the three orders of magnitude within our chain. 

    \section{Extractive Text Summarization}
      \label{s:app_summarization}
    
      We consider the article\footnote{\url{http://www.slate.com/articles/health_and_science/science/2017/04/explaining_science_won_t_fix_information_illiteracy.html}} entitled \textit{Scientists, Stop Thinking Explaining Science Will Fix Things}. 
      In order to generate an 11-sentences summary, we build 11 features as follows.
      For each sentence, we compute its number of characters and its number of words. 
      Then, we apply a Porter stemmer \cite{NLTK} and count again the number of characters and words in each sentence. 
      In addition, we sum the tf-idf values of the words in each sentence and compute the average cosine distance to all other sentences.
      Finally, we compute the position of the sentence in the original article and generate binary features indicating positions 1–5.
      We end up with a feature matrix $\bA$ of size $11\times 64$.
    
      Next, we display 3 summaries of the article, the first one is drawn using Algorithm~\ref{alg:zonoSampling} while the 2 others are constructed by picking 11 sentences uniformly at random. In both cases, we sort the sentences in the order they originally appear.
      Samples from Algorithm~\ref{alg:zonoSampling} induce a better coverage of the article than uniform draws.

    \clearpage
    One summary drawn with Algorithm~\ref{alg:zonoSampling}:
      \begin{formal}If you are a scientist, this disregard for evidence probably drives you crazy.\end{formal} \vspace{-0.395cm}
      \begin{formal}So what do you do about it?\end{formal} \vspace{-0.395cm}
      \begin{formal}Across the country, science communication and advocacy groups report upticks in interest.\end{formal} \vspace{-0.395cm}
      \begin{formal}In 2010, Dan Kahan, a Yale psychologist, essentially proved this theory wrong.\end{formal} \vspace{-0.395cm}
      \begin{formal}If the deficit model were correct, Kahan reasoned, then people with increased scientific literacy, regardless of worldview, should agree with scientists that climate change poses a serious risk to humanity.\end{formal} \vspace{-0.395cm}
      \begin{formal}Scientific literacy, it seemed, increased polarization.\end{formal} \vspace{-0.395cm}
      \begin{formal}This lumps scientists in with the nebulous "left" and, as Daniel Engber pointed out here in Slate about the upcoming March for Science, rebrands scientific authority as just another form of elitism.\end{formal} \vspace{-0.395cm}
      \begin{formal}Is it any surprise, then, that lectures from scientists built on the premise that they simply know more (even if it's true) fail to convince this audience?\end{formal} \vspace{-0.395cm}
      \begin{formal}With that in mind, it may be more worthwhile to figure out how to talk about science with people they already know, through, say, local and community interactions, than it is to try to publish explainers on national news sites.\end{formal} \vspace{-0.395cm}
      \begin{formal}Goldman also said scientists can do more than just educate the public: The Union of Concerned Scientists, for example, has created a science watchdogteam that keeps tabs on the activities of federal agencies.\end{formal} \vspace{-0.395cm}
      \begin{formal}There's also a certain irony that, right here in this article, I'm lecturing scientists about what they might not know-in other words, I'm guilty of following the deficit model myself.\end{formal}

      \vspace{-0.3cm}
      Two summaries drawn uniformly at random:

      \begin{formal} If you consider yourself to have even a passing familiarity with science, you likely find yourself in a state of disbelief as the president of the United States calls climate scientists "hoaxsters" and pushes conspiracy theories about vaccines.\end{formal} \vspace{-0.395cm}
      \begin{formal} In fact, it's so wrong that it may have the opposite effect of what they're trying to achieve.\end{formal} \vspace{-0.395cm}
      \begin{formal} Respondents who knew more about science generally, regardless of political leaning, were better able to identify the scientific consensus-in other words, the polarization disappeared.\end{formal} \vspace{-0.395cm}
      \begin{formal} In fact, well-meaning attempts by scientists to inform the public might even backfire.\end{formal} \vspace{-0.395cm}
      \begin{formal} Psychologists, aptly, dubbed this the "backfire effect."\end{formal} \vspace{-0.395cm}
      \begin{formal} But if scientists are motivated to change minds-and many enrolled in science communication workshops do seem to have this goal-they will be sorely disappointed.\end{formal} \vspace{-0.395cm}
      \begin{formal} That's not to say scientists should return to the bench and keep their mouths shut.\end{formal} \vspace{-0.395cm}
      \begin{formal} Goldman also said scientists can do more than just educate the public: The Union of Concerned Scientists, for example, has created a science watchdogteam that keeps tabs on the activities of federal agencies.\end{formal} \vspace{-0.395cm}
      \begin{formal} But I'm learning to better challenge scientists' assumptions about how communication works.\end{formal} \vspace{-0.395cm}
      \begin{formal} It's very logical, and my hunch is that it comes naturally to scientists because most have largely spent their lives in school-whether as students, professors, or mentors-and the deficit model perfectly explains how a scientist learns science.\end{formal} \vspace{-0.395cm}
      \begin{formal} So in the spirit of doing better, I'll not just write this article but also take the time to talk to scientists in person about how to communicate science strategically and to explain why it matters.\end{formal}

      \begin{formal} And it's not just Trump-plenty of people across the political spectrum hold bizarre and inaccurate ideas about science, from climate change and vaccines to guns and genetically modified organisms.\end{formal} \vspace{-0.395cm}
      \begin{formal} It seems many scientists would take matters into their own hands by learning how to better communicate their subject to the masses.\end{formal} \vspace{-0.395cm}
      \begin{formal} I've always had a handful of intrepid graduate students, but now, fueled by the Trump administration's Etch A Sketch relationship to facts, record numbers of scientists are setting aside the pipette for the pen.\end{formal} \vspace{-0.395cm}
      \begin{formal} This is because the way most scientists think about science communication-that just explaining the real science better will help-is plain wrong.\end{formal} \vspace{-0.395cm}
      \begin{formal} Before getting fired up to set the scientific record straight, scientists would do well to first considerthe science of science communication.\end{formal} \vspace{-0.395cm}
      \begin{formal} If the deficit model were correct, Kahan reasoned, then people with increased scientific literacy, regardless of worldview, should agree with scientists that climate change poses a serious risk to humanity.\end{formal} \vspace{-0.395cm}
      \begin{formal} Scientific literacy, it seemed, increased polarization.\end{formal} \vspace{-0.395cm}
      \begin{formal} Presenting facts that conflict with an individual's worldview, it turns out, can cause people to dig in further.\end{formal} \vspace{-0.395cm}
      \begin{formal} I spoke with Gretchen Goldman, research director of the Union of Concerned Scientists' Center for Science and Democracy, which offers communication and advocacy workshops.\end{formal} \vspace{-0.395cm}
      \begin{formal} Communication that appeals to values, not just intellect, research shows, can be far more effective.\end{formal} \vspace{-0.395cm}
      \begin{formal} I hope they end up doing the same.\end{formal}}

\end{document}